\documentclass[11pt]{article}
\usepackage[T1]{fontenc}
\usepackage[utf8]{inputenc}
\usepackage{lmodern}
\usepackage[margin=1in]{geometry}
\usepackage{amsmath,amssymb,amsthm,mathtools,bm}
\usepackage{microtype}
\usepackage{enumitem,booktabs}
\usepackage[round,authoryear]{natbib}
\usepackage[hidelinks]{hyperref}
\usepackage{bookmark}
\usepackage{authblk}

\newtheorem{thm}{Theorem}[section]
\newtheorem{prop}[thm]{Proposition}
\newtheorem{lem}[thm]{Lemma}

\theoremstyle{definition}
\newtheorem{defn}[thm]{Definition}
\theoremstyle{remark}

\newenvironment{eqn}{\begin{equation}}{\end{equation}}
\numberwithin{equation}{section}
\newcommand{\R}{\mathbb R}

\newcommand{\E}{\mathbb E}
\newcommand{\PP}{\mathbb P}
\newcommand{\F}{\mathcal F}
\newcommand{\cL}{\mathcal L}
\newcommand{\ind}{\mathbf 1}
\newcommand{\topr}{^{\mathsf T}}

\DeclareMathOperator{\tr}{tr}

\DeclareMathOperator{\Var}{Var}

\DeclareMathOperator*{\argmax}{arg\,max}
\newcommand{\dto}{\xrightarrow{\mathrm d}}

\setlist[itemize]{leftmargin=*,itemsep=0.55em,topsep=0.5em}
\setlist[enumerate]{leftmargin=*,itemsep=0.35em}
\allowdisplaybreaks[2]
\author[1]{Prakhar Singhvi\thanks{singhviprakhar12@gmail.com}}
\author[2]{Yi Zou\thanks{yizou1104@gmail.com}}
\author[3]{Abhishek Bhattacharjee\thanks{abhishek@theabstractmath.com}}

\affil[1,2,3]{Abstract Math Institute}
\affil[1]{\texttt{singhviprakhar12@gmail.com}}
\affil[2]{\texttt{yizou1104@gmail.com}}
\affil[3]{\texttt{abhishek@theabstractmath.com}}
\date{}
\title{Exact Bayes Regret and Asymptotic Optimality in High-Dimensional Gaussian Bandits}
\begin{document}
\maketitle
\begin{abstract}
We study Bayesian linear bandits with an isotropic Gaussian parameter, independent Gaussian candidate arms, and Gaussian reward noise when the horizon is proportional to the dimension. The normalized posterior uncertainty has an explicit limit that is uniform over all causal policies. Gaussian posterior identities then determine the limiting parameter overlaps without an assumed closure of the adaptive recursion. These results yield exact regret curves for Thompson sampling, posterior-mean greedy selection, and a family of policies that scale the posterior sampling covariance. The normalized realized cumulative regret converges in $L^1$, uniformly on compact proportional-time intervals. A policy-uniform lower bound identifies the limiting optimal Bayes regret and proves that posterior-mean greedy selection attains it. Thompson sampling incurs a strictly larger leading regret; its instantaneous regret ratio relative to greedy selection lies between one and two and approaches two at long proportional horizons. Closed-form cumulative curves also identify a different comparison in the vanishing-noise limit. Finally, the instantaneous regret converges to a nondegenerate Gaussian decision-loss distribution, rather than to its mean. The analysis separates the amount of information acquired by a bandit policy from the quality of the decisions made using that information.
\end{abstract}
\noindent\textbf{Keywords:} Bayesian bandit; Thompson sampling; greedy policy; exact regret; proportional asymptotics; posterior geometry; asymptotic Bayes optimality.

\section{Introduction}\label{sec:intro}
The exploration--exploitation tradeoff is usually expressed through an inequality for cumulative regret. Such an inequality can establish a rate without identifying the leading loss, the trajectory of learning, or the cost of a particular form of randomization. In a high-dimensional experiment, these distinctions remain relevant even when the horizon and the dimension have the same order.

This paper studies a Gaussian contextual bandit in the proportional regime. Each round presents a fixed number of independent, isotropic Gaussian candidate arms. The unknown coefficient has an isotropic Gaussian prior, and the observation noise is Gaussian. The model is sufficiently structured to permit exact calculations, but the collected design remains adaptive: an action may depend on every previous candidate, reward, and randomization.

The main statistical fact is a policy-independent first-order posterior uncertainty. Selecting one vector from a fixed Gaussian pool changes its conditional distribution, but the resulting change is too small to alter normalized spectral traces at a proportional horizon. This fact is uniform over policies, including policies that use the entire candidate pool through a nonlinear rule. A scalar posterior variance function consequently describes the leading estimation error under every causal policy.

The decision-theoretic consequences are not policy-independent. A posterior-mean greedy action maximizes the conditional expected reward at the current round. Thompson sampling first draws a new coefficient from the posterior and maximizes the sampled reward instead. In general bandit models, the latter randomization can improve future information. In the present asymptotic experiment, no policy changes the leading posterior uncertainty, so a first-order improvement in future learning cannot compensate for the immediate loss caused by the additional posterior draw.

We make this distinction exact. Posterior geometry determines the limiting correlation between a selected score and the true reward. Gaussian maxima then convert that correlation into the expected regret. The resulting formulas yield a continuum of regret curves, an optimal Bayes value over all causal policies, and closed-form comparisons between Thompson sampling and greedy selection. The statements concern the specified proportional Gaussian model; they are not a general comparison of these algorithms over arbitrary contextual distributions.

An additional distinction concerns the meaning of an instantaneous regret limit. With a fixed number of random candidates, the Gaussian maxima at a single round remain random even as the dimension diverges. Their expected difference converges to a deterministic curve, whereas the regret itself has a nondegenerate limiting distribution. Cumulative averaging removes this round-specific variation. The paper proves both conclusions separately.

\section{Existing work}\label{sec:related}
This section places exact proportional regret within the literatures on posterior sampling, exploration-free contextual bandits, and high-dimensional inference.

\paragraph{Thompson sampling and information-based regret analysis.}
\citet{AG2013} develop regret guarantees for a Gaussian-sampling algorithm with linear contextual payoffs. The general posterior-sampling principle and its applications are reviewed by \citet{Russo2018}. \citet{RV2016} bound Bayesian regret using the relationship between information gain and instantaneous loss. These results explain why randomized actions can be effective across broad classes of sequential decisions. The present paper studies a more specific question: the exact leading Bayes regret of genuine Gaussian posterior sampling when the dimension and horizon grow proportionally. A spectral calculation identifies the information acquired in this model, and a separate Gaussian decision calculation identifies its conversion into reward.

\paragraph{High-dimensional reward and diffusion limits.}
\citet{DM2013} analyze Gaussian-prior linear bandits in the data-poor regime and obtain cumulative reward bounds matching up to constants under geometric assumptions on a fixed action set. Their reward analysis already relates posterior uncertainty to the posterior-mean norm. Our use of that Gaussian identity is standard; the additional step is an exact, policy-uniform trace limit for fresh finite candidate pools, which determines a full proportional regret curve. A different asymptotic approach is developed by \citet{KuangWager2023} and \citet{FanGlynn2021}: shrinking arm gaps lead to diffusion descriptions of sequential experiments and posterior-sampling dynamics. Those small-gap limits retain stochastic state trajectories. Here the growing parameter dimension produces deterministic normalized posterior states, while individual candidate-pool losses remain random.

\paragraph{Greedy algorithms under contextual variation.}
The possibility that contexts provide sufficient exploration is established in several settings. \citet{Bastani2021} analyze mostly exploration-free algorithms under covariate diversity. \citet{Kannan2018} prove guarantees for greedy selection in a smoothed contextual model. \citet{KimOh2024} obtain polylogarithmic regret guarantees under local anti-concentration conditions. These works concern assumptions and performance bounds that differ from the present isotropic Gaussian, common-parameter experiment. Our result does not introduce the general principle that random contexts can make explicit exploration unnecessary. It identifies the exact proportional Bayes value, uniformly over competing policies, in a model where that principle can be quantified without unspecified leading constants.

\paragraph{Bayesian comparisons with greedy selection.}
\citet{Raghavan2023} establish near-dominance properties of batched greedy algorithms in a Bayesian smoothed contextual setting. Their comparisons involve batching and polylogarithmic factors or time rescaling. This is especially close in decision-theoretic motivation to the present work. Here the Gaussian model permits a same-horizon limiting comparison: the infimum over all causal policies has a specific deterministic value, unbatched posterior-mean greedy selection attains it to first order, and the excess loss of Thompson sampling is explicit. These conclusions are narrower in distributional scope and sharper in asymptotic constants; they should not be read as distribution-free dominance statements.

\paragraph{High-dimensional spectral and inferential methods.}
The independent-design spectral law originates with \citet{MP1967}; see \citet{BaiSilverstein2010} for a systematic treatment. \citet{DW2018} use spectral transforms to obtain exact high-dimensional limits for regularized prediction. The adaptive comparison used here follows from Gaussian transportation \citep{Talagrand1996} and matrix norm inequalities \citep{Bhatia1997}. We include the comparison argument in the appendix, so the regret theory does not assume an unproved spectral limit for Thompson sampling. In a different action space, \citet{Fan2025} analyze the covariance geometry of LinUCB on a Euclidean unit ball. The distinction between a continuously controlled action space and a fixed fresh Gaussian pool is substantive: the latter limits how much selection can alter the full-dimensional design.

\section{Our contributions}\label{sec:contributions}
The analysis gives a complete first-order decision theory for the proportional Gaussian bandit experiment, from posterior learning to the optimal achievable Bayes regret.
\begin{itemize}
\item \textbf{A posterior state law valid uniformly over all causal policies.}
Theorem~\ref{thm:geometry} identifies the posterior covariance trace, posterior-mean norm, estimation error, and overlap with the true parameter. The scalar function in Definition~\ref{def:limits} is explicit and satisfies a closed differential equation. No policy-specific stability assumption or postulated state-evolution recursion is needed. Uniformity is essential: it permits optimization over policies that themselves vary with dimension and horizon.
\item \textbf{Exact cumulative regret for a continuum of bandit policies.}
Theorem~\ref{thm:regret} gives the leading Bayes regret and the $L^1$ limit of realized cumulative regret for every fixed posterior covariance multiplier. Thompson sampling and posterior-mean greedy selection are the cases $\eta=1$ and $\eta=0$. The result is uniform over compact proportional-time intervals and supplies deterministic regret curves, rather than upper bounds with unspecified constants.
\item \textbf{The optimal Bayes value over the entire causal policy class.}
Theorem~\ref{thm:optimal} evaluates the asymptotic infimum of Bayes regret over all admissible policies, not only over linear-score rules or posterior-sampling policies. Posterior-mean greedy selection attains this value. The proof combines an exact one-step reward bound with policy-uniform posterior learning, thereby controlling both present decisions and their possible future informational benefit.
\item \textbf{An explicit first-order price of posterior randomization.}
Theorem~\ref{thm:comparison} proves strict monotonicity of the regret curve in the posterior covariance multiplier. For Thompson sampling, the instantaneous ratio to greedy regret is exactly $1+q(s)$, with $q$ defined in Definition~\ref{def:limits}; the cumulative ratio is strictly between one and two at every positive finite proportional horizon. This quantifies a performance difference that rate bounds alone do not distinguish.
\item \textbf{Closed-form curves and distinct horizon--noise regimes.}
Theorem~\ref{thm:closed} integrates the greedy and Thompson curves explicitly. Theorems~\ref{thm:horizon} and~\ref{thm:noise} identify their short-horizon expansions, long-horizon logarithmic constants, and vanishing-noise limits. The Thompson-to-greedy ratio approaches two in the long proportional-horizon limit at positive noise, but becomes $3/2$ after the noiseless interpolation point in the sequential vanishing-noise limit. The order of limits is stated explicitly.
\item \textbf{A distributional description of single-round bandit loss.}
Theorem~\ref{thm:distribution} identifies the nondegenerate instantaneous regret distribution and the limiting probability of selecting the truly best arm. This separates a deterministic expected regret curve from the residual Gaussian decision uncertainty at an individual round and prevents an incorrect pointwise concentration interpretation.
\end{itemize}
\paragraph{Organization of the paper.}
Section~\ref{sec:model} defines the experiment, policies, loss, and deterministic functions. Section~\ref{sec:main} states the posterior, regret, optimality, comparison, and distributional results. Section~\ref{sec:scaling} gives closed forms and horizon--noise asymptotics. Section~\ref{sec:discussion} discusses the scope and limitations of the conclusions. Appendix~\ref{app:spectral} proves the adaptive spectral comparison. Appendices~\ref{app:geometry}--\ref{app:scaling} contain all remaining proofs.

\section{The Gaussian bandit experiment}\label{sec:model}
This section specifies the information structure and defines every policy and loss used in the results.

\begin{defn}[Experiment and admissible policies]\label{def:model}
Fix an integer $K\geq2$ and constants $\tau,\sigma>0$. For each dimension $d$, let $\theta^\star\sim N(0,\tau^2I_d)$. Independently of this parameter, let the vectors $x_{t,a}\sim N(0,I_d/d)$, the noises $\varepsilon_t\sim N(0,\sigma^2)$, and policy random seeds $U_t$ be mutually independent over their indices. At round $t$, the learner observes the current candidates, selects $A_t\in\{1,\ldots,K\}$, and observes $Y_t=x_{t,A_t}\topr\theta^\star+\varepsilon_t$.

The initial observed history $\F_0$ is trivial. Subsequently, $\F_t$ is generated by all seeds, candidates, actions, and rewards through round $t$. An admissible policy chooses $A_t$ measurably from $\F_{t-1}$, $U_t$, and the current candidate pool. It has no further access to $\theta^\star$ or future variables. Histories and seeds take values in standard Borel spaces. The set of these policies is $\Pi_{d,K}$; its elements may depend on the horizon. Expectations integrate the prior, candidates, noises, and policy randomization.

Write $x_t=x_{t,A_t}$, $G_t=\sum_{j=1}^t x_jx_j\topr$, and
\begin{eqn}\label{eq:posterior}
C_t=(\tau^{-2}I_d+\sigma^{-2}G_t)^{-1},
\qquad
m_t=\sigma^{-2}C_t\sum_{j=1}^t x_jY_j,
\end{eqn}
with $G_0=0$, $C_0=\tau^2I_d$, and $m_0=0$.
\end{defn}
The normalization makes the unconditional reward variance of an unselected arm equal to $\tau^2$. The number of candidates is fixed while the dimension grows.

\begin{defn}[Posterior covariance multiplier and random selection]\label{def:policies}
For $\eta\geq0$, let $Z_t\sim N(0,I_d)$ be fresh policy randomization, and put
\begin{eqn}\label{eq:temperature}
B_t^{(\eta)}=m_{t-1}+\sqrt\eta\,C_{t-1}^{1/2}Z_t,
\qquad
u_t^{(\eta)}=
\begin{cases}
B_t^{(\eta)}/\|B_t^{(\eta)}\|,&B_t^{(\eta)}\ne0,\\
e_1,&B_t^{(\eta)}=0.
\end{cases}
\end{eqn}
The policy $\pi_\eta$ selects $A_t=\argmax_{a\leq K}x_{t,a}\topr\nu_t^{(\eta)}$, with ties resolved by the smallest index. The vector $Z_t$ is generated before the current candidates and independently of the parameter and experiment conditional on the observed past. The policies $\pi_0$ and $\pi_1$ are called posterior-mean greedy selection and Thompson sampling, respectively. A random-selection policy, denoted by $\pi_{\mathrm{rand}}$, selects an index uniformly and independently of the current candidates and observed past.
\end{defn}
At a zero posterior mean, every action maximizes the posterior expected reward. The convention in \eqref{eq:temperature} therefore remains a greedy action. Theorem~\ref{thm:geometry} verifies that $\pi_1$ uses the exact posterior, rather than a surrogate sampling distribution.

\begin{defn}[Regret and optimal Bayes value]\label{def:regret}
For a policy $\pi\in\Pi_{d,K}$, define
\[
r_t^\pi=\max_{a\leq K}x_{t,a}\topr\theta^\star-x_{t,A_t}\topr\theta^\star,
\quad L_T^\pi=\sum_{t=1}^T r_t^\pi,
\quad R_{d,T}^\pi=\E L_T^\pi,
\quad R_{d,T}^{\star}=\inf_{\pi\in\Pi_{d,K}}R_{d,T}^\pi.
\]
The benchmark knows $\theta^\star$ and maximizes the conditional mean reward in each current pool. Noise is not included in the regret difference. Superscripts $\eta$ and $\mathrm{rand}$ denote the policies in Definition~\ref{def:policies}.
\end{defn}

\begin{defn}[Deterministic state and regret curves]\label{def:limits}
For independent standard normal variables $G_1,\ldots,G_K$, let $\mu_K=\E\max_{a\leq K}G_a$. For $s\geq0$, set
\begin{eqn}\label{eq:vq}
\begin{split}
v(s)&=\frac{\sqrt{\{\sigma^2+\tau^2(s-1)\}^2+4\sigma^2\tau^2}
-\{\sigma^2+\tau^2(s-1)\}}{2},\\
q(s)&=\sqrt{1-v(s)/\tau^2}.
\end{split}
\end{eqn}
For $s>0$ and $\eta\geq0$, define
\begin{eqn}\label{eq:curves}
c_\eta(s)=\frac{\tau^2-v(s)}{\tau\sqrt{\tau^2+(\eta-1)v(s)}},
\qquad \rho_\eta(s)=\mu_K\tau\{1-c_\eta(s)\},
\qquad \mathcal R_\eta(\Gamma)=\int_0^\Gamma\rho_\eta(s)\,ds.
\end{eqn}
Set $c_\eta(0)=0$ and $\rho_\eta(0)=\mu_K\tau$, the continuous extensions at zero. For random selection, set $\rho_{\mathrm{rand}}(s)=\mu_K\tau$ and $\mathcal R_{\mathrm{rand}}(\Gamma)=\mu_K\tau\Gamma$.
\end{defn}

\begin{defn}[Gaussian single-round loss]\label{def:decision}
For $0\leq c<1$, let $(G_a,H_a)_{a=1}^K$ be independent pairs of independent standard normal variables, let $U_a=cG_a+\sqrt{1-c^2}H_a$, and let $J=\argmax_{a\leq K}G_a$. Define
\[
D_K(c)=\max_{a\leq K}U_a-U_J.
\]
Let $\phi_2(u,v;c)$ and $\Phi_2(u,v;c)$ be the density and distribution function of a centered bivariate normal vector with unit marginal variances and correlation $c$. Put
\begin{eqn}\label{eq:agreement}
p_K(c)=K\int_{\R^2}\phi_2(u,v;c)\Phi_2(u,v;c)^{K-1}\,du\,dv.
\end{eqn}
\end{defn}

\section{Main results}\label{sec:main}
This section identifies posterior learning, derives policy-specific decisions, and then optimizes over the full policy class.

\subsection{Posterior geometry without a policy-specific closure}
\begin{thm}[Policy-uniform posterior state]\label{thm:geometry}
Under Definition~\ref{def:model}, for every $t$ and every admissible policy,
$\cL(\theta^\star\mid\F_t)=N(m_t,C_t)$ and $0\prec C_t\preceq\tau^2I_d$. For every $\Gamma<\infty$,
\begin{eqn}\label{eq:trace-uniform}
\sup_{\pi\in\Pi_{d,K}}\E\sup_{0\leq s\leq\Gamma}
\left|d^{-1}\tr C_{\lfloor sd\rfloor}-v(s)\right|\longrightarrow0.
\end{eqn}
For every deterministic sequence $n_d/d\to s\geq0$, the following convergences hold in $L^1$, uniformly over $\pi\in\Pi_{d,K}$:
\begin{eqn}\label{eq:state}
\begin{aligned}
d^{-1}\|m_{n_d}\|^2&\longrightarrow\tau^2-v(s),&
d^{-1}(\theta^\star)\topr m_{n_d}&\longrightarrow\tau^2-v(s),\\
d^{-1}\|\theta^\star-m_{n_d}\|^2&\longrightarrow v(s),&
d^{-1}\|\theta^\star\|^2&\longrightarrow\tau^2.
\end{aligned}
\end{eqn}
For $s>0$, the cosine between $m_{n_d}$ and $\theta^\star$ converges in probability to $q(s)$; its value on the event $m_{n_d}=0$ may be set to zero. The function $v$ is strictly decreasing and satisfies
\begin{eqn}\label{eq:ode}
v(0)=\tau^2,\qquad
v'(s)=-\frac{\tau^2v(s)}{2v(s)+\sigma^2+\tau^2(s-1)}.
\end{eqn}
For the policy in Definition~\ref{def:policies}, at $t_d=n_d+1$,
\begin{eqn}\label{eq:sample-state}
\begin{split}
d^{-1}\|B_{t_d}^{(\eta)}\|^2&\longrightarrow\tau^2+(\eta-1)v(s),\\
d^{-1}(\theta^\star)\topr B_{t_d}^{(\eta)}&\longrightarrow\tau^2-v(s).
\end{split}
\end{eqn}
Both convergences are in $L^1$. If $s>0$, the cosine between $B_{t_d}^{(\eta)}$ and $\theta^\star$ converges in probability to $c_\eta(s)$. The same assertion holds for $s=0$ when the direction is interpreted as in \eqref{eq:temperature}.
\end{thm}
In particular, the posterior-mean cosine is $q(s)$, whereas the cosine between the true parameter and an independent posterior draw is $q(s)^2$. The distinction reflects posterior sampling variability, not a different leading posterior covariance.

\subsection{Exact one-step identities}
\begin{prop}[Reward and direction]\label{prop:one-step}
Under Definition~\ref{def:model}, let a unit direction $u_t$ be chosen from the observed past and independent policy randomization before the current candidate pool, and select $A_t=\argmax_{a\leq K}x_{t,a}\topr u_t$. Then
\begin{eqn}\label{eq:direction-regret}
\E(r_t\mid\theta^\star,\F_{t-1},u_t)
=\frac{\mu_K}{\sqrt d}\{\|\theta^\star\|-(\theta^\star)\topr u_t\}.
\end{eqn}
For every $\pi\in\Pi_{d,K}$, including policies not determined by a preselected direction,
\begin{eqn}\label{eq:reward-bound}
\E(x_{t,A_t}\topr\theta^\star\mid\F_{t-1})
\leq\frac{\mu_K}{\sqrt d}\|m_{t-1}\|.
\end{eqn}
Equality in \eqref{eq:reward-bound} holds for posterior-mean greedy selection. Consequently,
\begin{eqn}\label{eq:greedy-exact}
\E r_t^0=\frac{\mu_K}{\sqrt d}
\{\E\|\theta^\star\|-\E\|m_{t-1}\|\},
\qquad
\E r_t^{\mathrm{rand}}=\frac{\mu_K}{\sqrt d}\E\|\theta^\star\|.
\end{eqn}
\end{prop}
The bound in \eqref{eq:reward-bound} holds under each policy's own posterior history. It does not compare the posterior means of different policies at finite dimension.

\subsection{Exact regret curves and optimality}
\begin{thm}[Expected and realized regret]\label{thm:regret}
Under Definitions~\ref{def:model}--\ref{def:limits}, for every fixed $\eta\geq0$ and every $s\geq0$,
\begin{eqn}\label{eq:instant-mean}
\E r_{\lfloor sd\rfloor+1}^\eta\longrightarrow\rho_\eta(s).
\end{eqn}
For every $\Gamma<\infty$,
\begin{eqn}\label{eq:realized-regret}
\E\sup_{0\leq\gamma\leq\Gamma}
\left|\frac{L_{\lfloor\gamma d\rfloor}^\eta}{d}-\mathcal R_\eta(\gamma)\right|
\longrightarrow0.
\end{eqn}
In particular,
$\sup_{0\leq\gamma\leq\Gamma}|d^{-1}R_{d,\lfloor\gamma d\rfloor}^\eta-\mathcal R_\eta(\gamma)|\to0$.
The same assertions hold for random selection with the functions in Definition~\ref{def:limits}.
For $s\geq0$,
\begin{eqn}\label{eq:special-curves}
\rho_0(s)=\mu_K\tau\{1-q(s)\},
\qquad
\rho_1(s)=\frac{\mu_K}{\tau}v(s)
=\mu_K\tau\{1-q(s)^2\}.
\end{eqn}
\end{thm}

\begin{thm}[Optimal first-order Bayes regret]\label{thm:optimal}
Under Definitions~\ref{def:model}--\ref{def:limits}, for every fixed $\Gamma>0$,
\begin{eqn}\label{eq:optimal}
\lim_{d\to\infty}\frac{R_{d,\lfloor\Gamma d\rfloor}^{\star}}{d}
=\mathcal R_0(\Gamma),
\qquad
\frac{R_{d,\lfloor\Gamma d\rfloor}^{0}-R_{d,\lfloor\Gamma d\rfloor}^{\star}}{d}
\longrightarrow0.
\end{eqn}
\end{thm}
The infimum in \eqref{eq:optimal} ranges over the full causal class of Definition~\ref{def:model}. The theorem asserts asymptotic Bayes optimality to first order, not finite-dimensional optimality or frequentist minimaxity.

\begin{thm}[The cost of posterior randomization]\label{thm:comparison}
Under Definition~\ref{def:limits}, for $s>0$ and $\eta\geq0$,
\begin{eqn}\label{eq:temperature-monotonicity}
\frac{\partial\rho_\eta(s)}{\partial\eta}
=\frac{\mu_K\{\tau^2-v(s)\}v(s)}
{2\{\tau^2+(\eta-1)v(s)\}^{3/2}}>0.
\end{eqn}
For every $\Gamma>0$,
\[
0<\mathcal R_0(\Gamma)<\mathcal R_1(\Gamma)
<\mathcal R_{\mathrm{rand}}(\Gamma),
\qquad
\mathcal R_0(\Gamma)<\mathcal R_1(\Gamma)<2\mathcal R_0(\Gamma).
\]
The instantaneous ratio and cumulative excess are
\begin{eqn}\label{eq:price}
\frac{\rho_1(s)}{\rho_0(s)}=1+q(s)\quad(s\geq0),
\qquad
\mathcal R_1(\Gamma)-\mathcal R_0(\Gamma)
=\mu_K\tau\int_0^\Gamma q(s)\{1-q(s)\}\,ds.
\end{eqn}
\end{thm}
The leading regret difference is positive even though the policies have the same leading posterior uncertainty. Theorem~\ref{thm:optimal} shows that this is a decision loss relative to the optimal value of the specified asymptotic experiment.

\subsection{Single-round fluctuations and best-arm agreement}
\begin{thm}[Limiting Gaussian decision loss]\label{thm:distribution}
Under Definitions~\ref{def:model}--\ref{def:decision}, for each fixed $s\geq0$ and $\eta\geq0$,
\begin{eqn}\label{eq:loss-distribution}
r_{\lfloor sd\rfloor+1}^\eta\dto\tau D_K\{c_\eta(s)\},
\qquad
\PP\left(A_{\lfloor sd\rfloor+1}
=\argmax_{a\leq K}x_{\lfloor sd\rfloor+1,a}\topr\theta^\star\right)
\longrightarrow p_K\{c_\eta(s)\}.
\end{eqn}
For $0\leq c<1$,
\[
\E D_K(c)=\mu_K(1-c),\qquad
\PP\{D_K(c)=0\}=p_K(c)\in(0,1),\qquad p_K(0)=1/K.
\]
In particular, for every finite $s$ and every $\sigma>0$, the limiting distribution in \eqref{eq:loss-distribution} is nondegenerate.
\end{thm}
The expected limit in \eqref{eq:instant-mean} is therefore not a probability limit of the individual regret to a deterministic constant.

\section{Closed forms and horizon--noise asymptotics}\label{sec:scaling}
This section evaluates the regret integrals and specifies how the limiting curves behave when their horizon or noise arguments vary.

\begin{thm}[Closed-form cumulative regret]\label{thm:closed}
Under Definition~\ref{def:limits}, for every $\Gamma\geq0$,
\begin{eqn}\label{eq:TS-closed}
\mathcal R_1(\Gamma)=\frac{\mu_K}{\tau}
\left\{\sigma^2\log\frac{\tau^2}{v(\Gamma)}
+\frac{\tau^4-v(\Gamma)^2}{2\tau^2}\right\},
\end{eqn}
and
\begin{eqn}\label{eq:greedy-closed}
\mathcal R_0(\Gamma)=\mu_K\tau
\left\{q(\Gamma)^2-\frac23q(\Gamma)^3
+\frac{\sigma^2}{\tau^2}
\left[\operatorname{arctanh}q(\Gamma)
-\frac{q(\Gamma)}{1+q(\Gamma)}\right]\right\}.
\end{eqn}
Here $\operatorname{arctanh}(q)=\tfrac12\log\{(1+q)/(1-q)\}$ for $0\leq q<1$.
\end{thm}

\begin{thm}[Short and long proportional horizons]\label{thm:horizon}
For fixed $\tau,\sigma>0$, as $\Gamma\downarrow0$,
\begin{eqn}\label{eq:small-horizon}
\begin{split}
\mathcal R_0(\Gamma)
&=\mu_K\tau\Gamma
-\frac{2\mu_K\tau^2}{3\sqrt{\tau^2+\sigma^2}}\Gamma^{3/2}
+O(\Gamma^{5/2}),\\
\mathcal R_1(\Gamma)
&=\mu_K\tau\Gamma
-\frac{\mu_K\tau^3}{2(\tau^2+\sigma^2)}\Gamma^2
+O(\Gamma^3).
\end{split}
\end{eqn}
For every fixed $\eta\geq0$, as $\Gamma\to\infty$,
\begin{eqn}\label{eq:large-horizon}
\mathcal R_\eta(\Gamma)
=\frac{\mu_K(1+\eta)\sigma^2}{2\tau}\log\Gamma+O(1),
\qquad
\frac{\mathcal R_\eta(\Gamma)}{\mathcal R_0(\Gamma)}\longrightarrow1+\eta.
\end{eqn}
The limits in this theorem concern the deterministic curves obtained after the proportional dimension limit. They do not assert a joint approximation for arbitrary sequences $T/d\to\infty$.
\end{thm}

\begin{thm}[Sequential vanishing-noise limit]\label{thm:noise}
Fix $\tau>0$ and $\Gamma\geq0$, and display the dependence on $\sigma$ by writing $v_\sigma$ and $\mathcal R_{\eta,\sigma}$. After the dimension limit in Theorem~\ref{thm:regret}, let $\sigma\downarrow0$. Then, for every $s\geq0$,
\begin{eqn}\label{eq:noise-state}
v_\sigma(s)\longrightarrow\tau^2(1-s)_+,
\quad
\rho_{1,\sigma}(s)\longrightarrow\mu_K\tau(1-s)_+,
\quad
\rho_{0,\sigma}(s)\longrightarrow\mu_K\tau\{1-\sqrt{\min(s,1)}\}.
\end{eqn}
Writing $h=\min(\Gamma,1)$ within this statement,
\begin{eqn}\label{eq:noise-regret}
\mathcal R_{1,\sigma}(\Gamma)\longrightarrow\mu_K\tau(h-h^2/2),
\qquad
\mathcal R_{0,\sigma}(\Gamma)\longrightarrow\mu_K\tau\{h-(2/3)h^{3/2}\}.
\end{eqn}
For $\Gamma\geq1$, the limits are $\mu_K\tau/2$ and $\mu_K\tau/3$, respectively. In particular,
\[
\lim_{\Gamma\to\infty}\lim_{\sigma\downarrow0}
\frac{\mathcal R_{1,\sigma}(\Gamma)}{\mathcal R_{0,\sigma}(\Gamma)}=\frac32,
\qquad
\lim_{\sigma\downarrow0}\lim_{\Gamma\to\infty}
\frac{\mathcal R_{1,\sigma}(\Gamma)}{\mathcal R_{0,\sigma}(\Gamma)}=2.
\]
\end{thm}
The interpolation point $s=1$ appears in the vanishing-noise state through the rank of a proportional Gaussian design. It is not an asserted spectral outlier transition for the adaptive bandit matrix.

\section{Discussion}\label{sec:discussion}
The results separate two components of a sequential decision problem. Posterior uncertainty describes what the observations have learned. The direction used to choose an action describes how that knowledge is converted into reward. In the present proportional Gaussian experiment, the first component has a universal leading limit, whereas the second depends on the policy. This separation makes an exact optimization over causal policies possible.

The optimality of greedy selection rests on the combined assumptions of the model. The prior is isotropic Gaussian, the candidates are independent isotropic Gaussian vectors, their number is fixed, and the horizon is proportional to dimension. Every policy receives a fresh candidate pool. A policy cannot repeatedly query an arbitrary direction in a continuous action set. These restrictions explain why deliberate selection cannot change the leading average posterior uncertainty. Correlated candidate pools, non-isotropic priors, growing pools outside the applicable spectral comparison regime, and other horizon scales can alter the informational value of actions.

The state theorem is Bayesian. Conditional Gaussianity and concentration of the prior norm are used to convert a posterior trace into an overlap with the true parameter. No uniform-in-parameter frequentist conclusion follows automatically. Similarly, asymptotic Bayes optimality does not imply that greedy selection is the exact finite-horizon Bayes policy, nor does the comparison imply that posterior sampling is generally inferior in contextual bandits.

The instantaneous distribution and cumulative curve describe different levels of averaging. The former retains the random geometry of a fixed candidate pool; the latter averages that randomness over a growing number of rounds. Distinguishing them is necessary even when every posterior scalar state converges to a deterministic value.

Finally, the calculations use normalized spectral traces, not the leading eigenvector of the adaptive Gram matrix. A finite number of spectral outliers could coexist with the same posterior trace limit. No learning-induced outlier transition for Thompson sampling is required or established here. The regret conclusions follow from policy-uniform bulk learning and exact Gaussian decision identities.

\appendix
\section{Adaptive spectral comparison}\label{app:spectral}
This appendix proves the spectral input uniformly over policies; all other claims are proved in the subsequent appendices.

\begin{lem}[Comparison with an independent Gaussian design]\label{lem:coupling}
Under Definition~\ref{def:model}, fix $N\geq1$. There is a coupling of the selected design $X_t=[x_1\ \cdots\ x_t]$ with independent vectors $z_j\sim N(0,I_d/d)$, assembled as $Z_t=[z_1\ \cdots\ z_t]$, for which
\begin{eqn}\label{eq:coupling}
\E\|X_N-Z_N\|_F^2\leq\frac{2N\log K}{d},
\qquad
\E\max_{t\leq N}\frac{\|G_t-Z_tZ_t\topr\|_*}{d}
\leq\frac{2N}{d}\left\{\sqrt{\frac{2\log K}{d}}+\frac{\log K}{d}\right\}.
\end{eqn}
Here $\|\cdot\|_F$ and $\|\cdot\|_*$ are the Frobenius and nuclear norms. The constants are independent of the policy.
\end{lem}
\begin{proof}
Condition on the past and current policy seed, and let $\nu$ be the conditional law of $\sqrt d\,x_t$. If $\gamma_d$ is standard Gaussian measure, then for every Borel set $B$,
\[
\nu(B)=\sum_{a=1}^K\PP(\sqrt d\,x_{t,a}\in B,A_t=a\mid\text{past, seed})
\leq K\gamma_d(B).
\]
Thus $\nu\ll\gamma_d$ and its relative entropy satisfies
$\int\log(d\nu/d\gamma_d)\,d\nu\leq\log K$.
The Gaussian transportation inequality of \citet{Talagrand1996} provides a coupling with a standard Gaussian vector having expected squared distance at most $2\log K$. Rescaling gives a conditional cost at most $2\log K/d$.

The couplings may be chosen as measurable kernels. One way to see this is to use the measurable selection of optimal quadratic couplings on Euclidean spaces; alternatively, measurable couplings with arbitrarily small excess cost give the same bounds by passage to a limit. Apply the kernels successively along the experiment. Conditional on the enlarged past, the newly added comparison vector has the same $N(0,I_d/d)$ law. The comparison vectors are therefore mutually independent. The original experiment is preserved by disintegrating each coupling with respect to its selected-vector marginal. Fresh candidate pools remain independent of the previously added coupling variables. Summation of the conditional costs gives the first inequality in \eqref{eq:coupling}.

Write $D_t=X_t-Z_t$. The identity
$G_t-Z_tZ_t\topr=Z_tD_t\topr+D_tZ_t\topr+D_tD_t\topr$ and the Schatten norm inequality $\|AB\topr\|_*\leq\|A\|_F\|B\|_F$ imply
\[
\|G_t-Z_tZ_t\topr\|_*
\leq2\|Z_t\|_F\|D_t\|_F+\|D_t\|_F^2.
\]
Both Frobenius norms are nondecreasing with the number of columns. Apply Cauchy--Schwarz at $N$, use $\E\|Z_N\|_F^2=N$, and divide by $d$. This proves the second inequality. The matrix inequality is standard; see \citet{Bhatia1997}.
\end{proof}

\begin{lem}[Uniform posterior trace]\label{lem:trace}
Under Definition~\ref{def:model}, \eqref{eq:trace-uniform} holds, and $v$ in \eqref{eq:vq} satisfies
\begin{eqn}\label{eq:v-polynomial}
v(s)^2+\{\sigma^2+\tau^2(s-1)\}v(s)-\sigma^2\tau^2=0.
\end{eqn}
\end{lem}
\begin{proof}
For positive semidefinite matrices $A,B$, the eigenvalue perturbation bound in nuclear norm implies
\[
\left|\frac1d\tr f(A)-\frac1d\tr f(B)\right|
\leq\frac{\operatorname{Lip}(f)}d\|A-B\|_*.
\]
Use $f(x)=(\tau^{-2}+\sigma^{-2}x)^{-1}$, whose Lipschitz constant is at most $\tau^4/\sigma^2$. Taking $N=\lfloor\Gamma d\rfloor$ in Lemma~\ref{lem:coupling}, the difference between the selected-design trace and the independent-Gaussian trace converges to zero in expected supremum norm, uniformly over policies.

For fixed $s>0$, the empirical law of $Z_{\lfloor sd\rfloor}Z_{\lfloor sd\rfloor}\topr$ converges to the Marchenko--Pastur measure
\[
(1-s)_+\delta_0+
\frac{\sqrt{\{(1+\sqrt s)^2-x\}\{x-(1-\sqrt s)^2\}}}{2\pi x}
\ind_{[(1-\sqrt s)^2,(1+\sqrt s)^2]}(x)\,dx;
\]
see \citet{MP1967,BaiSilverstein2010}. The mean of this normalization is $s$. Since $f$ is bounded and continuous, its empirical integral converges in $L^1$ to the integral against this measure. The negative-real resolvent of the measure satisfies
$\lambda h^2+(\lambda+s-1)h-1=0$ for $\lambda>0$, where $h$ is the positive root. Taking $\lambda=\sigma^2/\tau^2$ and multiplying by $\sigma^2$ gives exactly $v(s)$ and \eqref{eq:v-polynomial}. At $s=0$, the trace is identically $\tau^2$.

Each Gaussian posterior trace path is nonincreasing in $s$, takes values in $[0,\tau^2]$, and converges at fixed times to the continuous function $v$. On a finite partition of $[0,\Gamma]$, monotonicity bounds its maximum discrepancy by the maximum discrepancy at partition points plus the modulus of continuity of $v$. First let $d\to\infty$, then refine the partition. This gives convergence in expected supremum norm for the Gaussian trace paths. Lemma~\ref{lem:coupling} transfers it to every policy with the same bound.
\end{proof}

\section{Posterior identities and geometry}\label{app:geometry}
\subsection{Proof of Theorem~\ref{thm:geometry}}
\begin{proof}
\textbf{Conditional posterior.}
Given the observed history, the candidate densities and policy kernels do not depend on the value assigned to $\theta^\star$. Their factors therefore cancel from Bayes' formula. The remaining likelihood is proportional to
\[
\exp\left\{-\frac{\|\theta\|^2}{2\tau^2}
-\frac1{2\sigma^2}\sum_{j=1}^t(Y_j-x_j\topr\theta)^2\right\}.
\]
Completing the square yields $N(m_t,C_t)$ with \eqref{eq:posterior}. This calculation remains valid when every past candidate and seed is retained. The matrix bound follows directly from $G_t\succeq0$. Lemma~\ref{lem:trace} proves \eqref{eq:trace-uniform}.

\textbf{Norms and posterior means.}
Conditional Gaussianity gives
\[
\E(\|\theta^\star\|^2\mid\F_n)=\|m_n\|^2+\tr C_n.
\]
The prior norm has mean $d\tau^2$ and variance $2d\tau^4$. Conditional expectation is an $L^2$ contraction, so
\begin{eqn}\label{eq:mean-norm-bound}
\E\left[\frac{\|m_n\|^2+\tr C_n}{d}-\tau^2\right]^2
\leq\frac{2\tau^4}{d}.
\end{eqn}
This bound is uniform in the policy and time. Combining it with Lemma~\ref{lem:trace} proves the first convergence in \eqref{eq:state}. Conditional on $\F_n$,
\[
\E\{(\theta^\star)\topr m_n\mid\F_n\}=\|m_n\|^2,
\qquad
\Var\{(\theta^\star)\topr m_n\mid\F_n\}
=m_n\topr C_nm_n\leq\tau^2\|m_n\|^2.
\]
Since $\E\|m_n\|^2\leq d\tau^2$, the normalized difference between this inner product and $\|m_n\|^2$ tends to zero in $L^2$, uniformly over policies. Similarly,
\[
\E(\|\theta^\star-m_n\|^2\mid\F_n)=\tr C_n,
\quad
\Var(\|\theta^\star-m_n\|^2\mid\F_n)=2\tr C_n^2\leq2d\tau^4.
\]
These inequalities prove the remaining convergences. The prior norm convergence follows directly from its variance. For $s>0$, the limiting posterior-mean norm is positive, so division by the limiting norms gives the cosine $q(s)$.

\textbf{Differential equation.}
The positive root of \eqref{eq:v-polynomial} is \eqref{eq:vq}, and
$2v(s)+\sigma^2+\tau^2(s-1)$ equals the strictly positive square root appearing there. Implicit differentiation gives \eqref{eq:ode}. In particular, $v(0)=\tau^2$, $v'(s)<0$, and $0<v(s)<\tau^2$ for $s>0$.

\textbf{Randomized score vectors.}
Given $\F_n$, represent $\theta^\star=m_n+e$ and $B_{n+1}^{(\eta)}=m_n+\sqrt\eta\,e'$, with $e,e'$ conditionally independent $N(0,C_n)$ vectors. The conditional first two moments give
\[
\begin{split}
\E(\|B_{n+1}^{(\eta)}\|^2\mid\F_n)&=\|m_n\|^2+\eta\tr C_n,\\
\Var(\|B_{n+1}^{(\eta)}\|^2\mid\F_n)&=4\eta m_n\topr C_nm_n+2\eta^2\tr C_n^2,\\
\E\{(\theta^\star)\topr B_{n+1}^{(\eta)}\mid\F_n\}&=\|m_n\|^2,\\
\Var\{(\theta^\star)\topr B_{n+1}^{(\eta)}\mid\F_n\}
&=(1+\eta)m_n\topr C_nm_n+\eta\tr C_n^2.
\end{split}
\]
Their expected variances are $O(d)$ for fixed $\eta$, uniformly in the policy and time. This proves \eqref{eq:sample-state}. For $s>0$, both limiting norms are positive and division gives $c_\eta(s)$.

It remains to justify the assertion at $s=0$ when $n_d/d\to0$, especially for $\eta=0$. For every pre-context unit direction $u$ measurable from $\F_{n_d}$ and independent policy randomization,
\[
\frac{|(\theta^\star)\topr u|}{\sqrt d}
\leq\frac{\|m_{n_d}\|}{\sqrt d}
+\frac{|(\theta^\star-m_{n_d})\topr u|}{\sqrt d}.
\]
The first term tends to zero in $L^2$ because $v(0)=\tau^2$. The second has second moment at most $\tau^2/d$. As $\|\theta^\star\|/\sqrt d\to\tau$, the cosine tends to zero. This includes the deterministic convention at a zero score vector and finishes the proof.
\end{proof}

\subsection{Proof of Proposition~\ref{prop:one-step}}
\begin{proof}
Conditional on $u_t$, write
$\sqrt d\,x_{t,a}=G_au_t+w_a$, where the $G_a$ are independent standard normal variables and the $w_a$ are independent centered Gaussian vectors on the orthogonal complement of $u_t$. The two collections are independent. Selection depends only on $G_1,\ldots,G_K$. Hence
\[
\E(x_{t,A_t}\mid u_t,\text{past})=\mu_Ku_t/\sqrt d.
\]
Conditional on $\theta^\star$, the true rewards of all candidates are independent centered Gaussians with variance $\|\theta^\star\|^2/d$. Their expected maximum is $\mu_K\|\theta^\star\|/\sqrt d$. Subtraction yields \eqref{eq:direction-regret}.

For an arbitrary admissible policy, the current pool and seed are independent of $\theta^\star$ conditionally on $\F_{t-1}$. Conditional on them, the expected selected reward is $x_{t,A_t}\topr m_{t-1}$, which is at most $\max_a x_{t,a}\topr m_{t-1}$. Integrating the pool gives \eqref{eq:reward-bound}. Greedy selection attains that maximum; when $m_{t-1}=0$, both sides are zero regardless of the tie convention. The expected maximum true reward remains $\mu_K\E\|\theta^\star\|/\sqrt d$. This proves the first identity in \eqref{eq:greedy-exact}. A random index has zero expected true reward, proving the second.
\end{proof}

\section{Regret limits, optimality, and comparisons}\label{app:regret}
\subsection{Proof of Theorem~\ref{thm:regret}}
\begin{proof}
\textbf{Conditional mean and pointwise limit.}
Reveal the true parameter for purposes of the proof and, at each round, reveal the fresh score-vector randomization before the candidates. Define
\[
g_t=\frac{\mu_K}{\sqrt d}\{\|\theta^\star\|-(\theta^\star)\topr\nu_t^{(\eta)}\}.
\]
By Proposition~\ref{prop:one-step}, this is the conditional expected regret after the direction is revealed and before the candidates are drawn. Theorem~\ref{thm:geometry} gives $g_{\lfloor sd\rfloor+1}\to\rho_\eta(s)$ in probability for every $s\geq0$. Moreover,
$0\leq g_t\leq2\mu_K\|\theta^\star\|/\sqrt d$, and the latter variables have uniformly bounded second moments. Consequently the convergence is in $L^1$, which proves \eqref{eq:instant-mean}. For $\eta=0,1$, substitution in \eqref{eq:curves} gives \eqref{eq:special-curves}.

\textbf{Martingale fluctuations.}
Let $\mathcal E_t=\F_t\vee\sigma(\theta^\star)$, including the policy seeds already specified in Definition~\ref{def:model}. Although $g_t$ also depends on the current seed, the tower property implies
$\E(r_t^\eta-g_t\mid\mathcal E_{t-1})=0$.
Thus $\sum_{t\leq n}(r_t^\eta-g_t)$ is a martingale for the end-of-round filtration. Since
\[
0\leq r_t^\eta\leq2\max_{a\leq K}|x_{t,a}\topr\theta^\star|,
\]
conditioning on $\theta^\star$ and using the Gaussian maximum of fixed size gives
$\sup_{d,t}\E(r_t^\eta)^2\leq4\tau^2\E\max_{a\leq K}|G_a|^2<\infty$.
The conditional-mean property also yields $\E(r_t^\eta-g_t)^2\leq\E(r_t^\eta)^2$. Doob's inequality therefore gives
\begin{eqn}\label{eq:martingale-bound}
\E\max_{n\leq\lfloor\Gamma d\rfloor}
\left|\frac1d\sum_{t=1}^n(r_t^\eta-g_t)\right|^2
\leq\frac{C_{K,\tau}\Gamma}{d},
\end{eqn}
where $C_{K,\tau}$ is finite and independent of $d$.

\textbf{Accumulated conditional loss.}
Define a piecewise constant function on $[0,\Gamma]$ by $\widetilde g_d(s)=g_{\lfloor sd\rfloor+1}$. Pointwise $L^1$ convergence and the uniform integrable bound above imply, by dominated convergence,
\[
\E\int_0^\Gamma|\widetilde g_d(s)-\rho_\eta(s)|\,ds\longrightarrow0.
\]
For any $\gamma\leq\Gamma$, the difference between
$d^{-1}\sum_{t\leq\lfloor\gamma d\rfloor}g_t$ and $\int_0^\gamma\widetilde g_d(s)\,ds$ is at most $2\mu_K\|\theta^\star\|/(d\sqrt d)$. Bounding the difference of their limiting primitives by the integral of the absolute difference gives convergence in expected supremum norm. Combining this with \eqref{eq:martingale-bound} proves \eqref{eq:realized-regret} and, by taking expectations, its Bayes-regret consequence.

For random selection the corresponding conditional mean is simply $\mu_K\|\theta^\star\|/\sqrt d$. The same moment and martingale bounds apply and prove the stated random-selection limits.
\end{proof}

\subsection{Proof of Theorem~\ref{thm:optimal}}
\begin{proof}
The pointwise bound \eqref{eq:reward-bound} implies, for every admissible policy,
\begin{eqn}\label{eq:policy-lower}
R_{d,T}^\pi\geq\mu_K\sum_{t=1}^T
\left\{\frac{\E\|\theta^\star\|}{\sqrt d}
-\frac{\E\|m_{t-1}^\pi\|}{\sqrt d}\right\}.
\end{eqn}
We justify uniform replacement of the posterior-mean norms. By \eqref{eq:mean-norm-bound}, the inequality $|\sqrt a-\sqrt b|\leq\sqrt{|a-b|}$ for nonnegative $a,b$, and Lemma~\ref{lem:trace},
\[
\begin{split}
&\sup_{\pi\in\Pi_{d,K}}\sup_{0\leq n\leq\lfloor\Gamma d\rfloor}
\E\left|\frac{\|m_n^\pi\|}{\sqrt d}-\sqrt{\tau^2-v(n/d)}\right|\\
&\quad\leq
\left\{\tau^2\sqrt{\frac2d}
+\sup_{\pi\in\Pi_{d,K}}\E\sup_{0\leq s\leq\Gamma}
\left|d^{-1}\tr C_{\lfloor sd\rfloor}^\pi-v(s)\right|\right\}^{1/2}
\longrightarrow0.
\end{split}
\]
Also $\E\|\theta^\star\|/\sqrt d\to\tau$. Divide \eqref{eq:policy-lower} by $d$, take $T=\lfloor\Gamma d\rfloor$, and use the Riemann sum for the continuous function $\sqrt{\tau^2-v(s)}$. The result is a lower bound $\mathcal R_0(\Gamma)-o(1)$ with an error independent of the policy. It therefore applies after taking the infimum over $\Pi_{d,K}$. The upper bound is the greedy limit in Theorem~\ref{thm:regret}. Together they prove both conclusions.
\end{proof}

\subsection{Proof of Theorem~\ref{thm:comparison}}
\begin{proof}
For $s>0$, $0<v(s)<\tau^2$. Differentiate \eqref{eq:curves} with respect to $\eta$ to obtain \eqref{eq:temperature-monotonicity}. Strict positivity follows from $\mu_K>0$ for $K\geq2$. The special formulas \eqref{eq:special-curves} and $0<q(s)<1$ give
\[
0<\rho_0(s)<\rho_1(s)<\mu_K\tau,
\qquad \rho_1(s)=\rho_0(s)\{1+q(s)\}.
\]
Integrating these strict inequalities over any interval of positive length proves the cumulative comparisons. Subtracting the two special curves gives the excess-regret integral in \eqref{eq:price}. At $s=0$, both instantaneous regrets equal $\mu_K\tau$ and the ratio identity remains valid.
\end{proof}

\section{Distribution of instantaneous regret}\label{app:distribution}
\subsection{Proof of Theorem~\ref{thm:distribution}}
\begin{proof}
For a fixed round, condition on $\theta^\star$ and the selected score direction $\nu_t^{(\eta)}$, before revealing the candidates. The normalized score and true reward of arm $a$ are
\[
\sqrt d\,x_{t,a}\topr\nu_t^{(\eta)},
\qquad
\frac{\sqrt d\,x_{t,a}\topr\theta^\star}{\|\theta^\star\|}.
\]
They are a standard bivariate Gaussian pair with correlation
$(\theta^\star)\topr\nu_t^{(\eta)}/\|\theta^\star\|$, independently across arms. Theorem~\ref{thm:geometry} gives convergence of this correlation to $c_\eta(s)$ and of $\|\theta^\star\|/\sqrt d$ to $\tau$. The finite array of scores and rewards therefore converges in distribution to the Gaussian array of Definition~\ref{def:decision}, with the reward scale $\tau$.

For $c<1$, the limiting scores and rewards have no ties almost surely. The selected-index and maximum-loss maps are continuous outside these tie sets. The continuous mapping theorem yields the first limit in \eqref{eq:loss-distribution}. The event that the two maximizers agree also has boundary contained in tie sets, proving convergence of its probability.

To compute that probability, fix a candidate with score $u$ and true reward $v$. Each other pair lies below it in both coordinates with probability $\Phi_2(u,v;c)$, independently. Integrating its density and summing over the $K$ possible common maximizers gives \eqref{eq:agreement}. At $c=0$, the two maximizing indices are independent and uniform, giving $p_K(0)=1/K$.

For the mean loss, the true rewards $U_a$ are standard normal, so $\E\max_a U_a=\mu_K$. Since $J$ depends only on the scores, $H_J$ has mean zero, and $\E U_J=c\E G_J=c\mu_K$. Hence $\E D_K(c)=\mu_K(1-c)$. Alternatively, uniform integrability follows from the same squared-maximum bound as in the proof of Theorem~\ref{thm:regret}.

Finally, for $0\leq c<1$, the joint density of each score--reward pair is strictly positive on all of $\R^2$. There are nonempty open sets of arrays on which a common candidate strictly maximizes both coordinates, and nonempty open sets on which the maximizers differ. Both events have positive probability. The first has loss zero, and the second has strictly positive loss. This proves $p_K(c)\in(0,1)$ and nondegeneracy. For positive noise and finite $s$, Definition~\ref{def:limits} gives $c_\eta(s)<1$, completing the argument.
\end{proof}

\section{Evaluation and asymptotic analysis of the curves}\label{app:scaling}
\subsection{Proof of Theorem~\ref{thm:closed}}
\begin{proof}
Solving \eqref{eq:v-polynomial} for $s$ gives the change of variables
\begin{eqn}\label{eq:time-from-v}
s=1-\frac{\sigma^2}{\tau^2}+\frac{\sigma^2}{v}-\frac v{\tau^2},
\qquad
\frac{ds}{dv}=-\frac{\sigma^2}{v^2}-\frac1{\tau^2}.
\end{eqn}
The variable $v$ decreases from $\tau^2$ to $v(\Gamma)$. Substitute \eqref{eq:time-from-v} into $\mathcal R_1=(\mu_K/\tau)\int_0^\Gamma v(s)\,ds$. Integration yields
\[
\int_0^\Gamma v(s)\,ds
=\sigma^2\log\frac{\tau^2}{v(\Gamma)}
+\frac{\tau^4-v(\Gamma)^2}{2\tau^2},
\]
which proves \eqref{eq:TS-closed}.

For greedy selection, substitute $v=\tau^2(1-q^2)$ into \eqref{eq:time-from-v} to obtain
\[
s=q^2+\frac{\sigma^2}{\tau^2}\frac{q^2}{1-q^2},
\qquad
\frac{ds}{dq}=2q+\frac{2\sigma^2q}{\tau^2(1-q^2)^2}.
\]
Multiplication by $1-q$ and integration from zero to $q(\Gamma)$ gives the polynomial term $q^2-(2/3)q^3$. For the remaining term, the derivative identity
\[
\frac{d}{dq}\left\{\operatorname{arctanh}q-\frac q{1+q}\right\}
=\frac{2q}{(1-q)(1+q)^2}
=\frac{2q(1-q)}{(1-q^2)^2}
\]
gives exactly \eqref{eq:greedy-closed}.
\end{proof}

\subsection{Proof of Theorem~\ref{thm:horizon}}
\begin{proof}
From \eqref{eq:ode},
$v'(0)=-\tau^4/(\tau^2+\sigma^2)$. The explicit square-root formula is smooth near zero, so
\[
v(s)=\tau^2-\frac{\tau^4}{\tau^2+\sigma^2}s+O(s^2),
\qquad
q(s)=\frac{\tau}{\sqrt{\tau^2+\sigma^2}}\sqrt s+O(s^{3/2}).
\]
Substitution into \eqref{eq:special-curves} and integration gives \eqref{eq:small-horizon}.

For large $s$, \eqref{eq:time-from-v} implies $v(s)=\sigma^2/s+O(s^{-2})$. Expanding \eqref{eq:curves} at $v=0$ gives
\[
\frac{\tau^2-v}{\sqrt{\tau^2+(\eta-1)v}}
=\tau-\frac{1+\eta}{2\tau}v+O(v^2),
\qquad
\rho_\eta(s)=\frac{\mu_K(1+\eta)\sigma^2}{2\tau s}+O(s^{-2}).
\]
The constants may depend on the fixed parameters and $\eta$. Integrate from any fixed positive lower endpoint to $\Gamma$. The integrable $O(s^{-2})$ remainder proves the first assertion in \eqref{eq:large-horizon}. Its leading coefficient for $\eta=0$ is strictly positive, so taking the ratio proves the second.
\end{proof}

\subsection{Proof of Theorem~\ref{thm:noise}}
\begin{proof}
The explicit expression \eqref{eq:vq} converges, as $\sigma\downarrow0$, to
$\{\tau^2|s-1|-\tau^2(s-1)\}/2=\tau^2(1-s)_+$.
The two special regret formulas give \eqref{eq:noise-state}. Since both instantaneous curves are bounded by $\mu_K\tau$, dominated convergence applies on every finite interval. Direct integration gives
\[
\int_0^\Gamma(1-s)_+\,ds=h-h^2/2,
\qquad
\int_0^\Gamma\{1-\sqrt{\min(s,1)}\}\,ds
=h-(2/3)h^{3/2},
\]
with $h$ as in the statement. This proves \eqref{eq:noise-regret} and the plateau values. Taking the ratios for $\Gamma\geq1$ gives the first iterated limit. The second follows from \eqref{eq:large-horizon} at every fixed positive $\sigma$.
\end{proof}

\bibliographystyle{plainnat}
\bibliography{references}
\end{document}